\documentclass[letterpaper, 10 pt, conference]{ieeeconf}
\IEEEoverridecommandlockouts   
\usepackage{xcolor}
\usepackage{cite}
\usepackage{graphicx}
\usepackage{epsfig} 
\usepackage{amssymb}  
\usepackage{url}            
\usepackage{amsmath} 
\usepackage{microtype}      
\usepackage{subcaption}
\usepackage{color}
\usepackage{bm}
\usepackage{siunitx}
\usepackage{cleveref}
\usepackage{cancel}
\graphicspath{{./figures/}}

\newcommand{\B}[1]{\mathbf{#1}}
\newtheorem{assumption}{Assumption}
\newtheorem{theorem}{Theorem}[section]

\newtheorem{lemma}[theorem]{Lemma}

\crefname{lemma}{lemma}{lemmas}
\Crefname{lemma}{Lemma}{Lemmas}
\crefname{thm}{theorem}{theorems}
\Crefname{thm}{Theorem}{Theorems}
\crefname{section}{Sec.}{sections}
\crefname{figure}{Fig.}{figures}
\crefformat{equation}{(#2#1#3)}
\crefrangeformat{equation}{(#3#1#4) to~(#5#2#6)}
\crefmultiformat{equation}{(#2#1#3)}%
{ and~(#2#1#3)}{, (#2#1#3)}{ and~(#2#1#3)}
\crefrangemultiformat{equation}{(#3#1#4) to~(#5#2#6)}%
{ and~(#3#1#4) to~(#5#2#6)}{, (#3#1#4) to~(#5#2#6)}{ and~(#3#1#4) to~(#5#2#6)}

\newcommand{\nn}{\emph{Neural-Lander}}
\newcommand{\baseline}{Baseline Controller}
\newcommand{\baselinelong}{Baseline Nonlinear Tracking Controller}

\title{\LARGE \bf
Neural Lander: Stable Drone Landing Control\\ Using Learned Dynamics
}

\author{Guanya Shi$^{*,1}$, Xichen Shi$^{*,1}$, Michael O'Connell$^{*,1}$, Rose Yu$^{2}$, Kamyar Azizzadenesheli$^{3}$, \\Animashree Anandkumar$^{1}$, Yisong Yue$^{1}$, and Soon-Jo Chung$^{1}$
\thanks{$^{*}$ These authors contributed equally to this work. $^{1}$California Institute of Technology, $^{2}$Northeastern University, $^{3}$University of California, Irvine. }
}

\begin{document}

\maketitle
\thispagestyle{empty}
\pagestyle{empty}

\begin{abstract}
Precise near-ground trajectory control  is difficult for multi-rotor drones, due to the complex aerodynamic effects caused by interactions between multi-rotor airflow and the environment. Conventional control methods often fail to properly account for these complex effects and fall short in accomplishing smooth landing. 
In this paper, we present a novel deep-learning-based robust nonlinear controller (\nn) that improves control performance of a quadrotor during landing. Our approach combines a nominal dynamics model with a Deep Neural Network (DNN) that learns high-order interactions. We apply spectral normalization (SN) to constrain the Lipschitz constant of the DNN. Leveraging this Lipschitz property, we design a nonlinear feedback linearization controller using the learned model and prove system stability with disturbance rejection. To the best of our knowledge, this is the first DNN-based nonlinear feedback controller with stability guarantees that can utilize arbitrarily large neural nets. 
Experimental results demonstrate that the proposed controller significantly outperforms a \baselinelong{} in both landing and cross-table trajectory tracking cases. 
We also empirically show that the DNN generalizes well to unseen data outside the training domain.
\end{abstract}

\section{Introduction}
Unmanned Aerial Vehicles (UAVs) require high precision control of aircraft positions, especially during landing and take-off. This problem is challenging largely due to complex interactions of rotor and wing airflows with the ground. The aerospace community has long identified such ground effect that can cause an increased lift force and a reduced aerodynamic drag. These effects can be both helpful and disruptive in flight stability~\cite{cheeseman1955effect}, and the complications are exacerbated with multiple rotors. Therefore, performing automatic landing of UAVs is risk-prone, and requires expensive high-precision sensors as well as carefully designed controllers. 

Compensating for ground effect is a long-standing problem in the aerial robotics community. Prior work has largely focused on mathematical modeling (e.g. \cite{nonaka2011integral}) as part of system identification (ID). These models are later used to approximate aerodynamics forces during flights close to the ground and combined with controller design for feed-forward cancellation  (e.g. \cite{danjun2015autonomous}). However, existing theoretical ground effect models are derived based on steady-flow conditions, whereas most practical cases exhibit unsteady flow. Alternative approaches, such as integral or adaptive control methods, often suffer from slow response and delayed feedback. \cite{Berkenkamp2016SafeOpt} employs Bayesian Optimization for open-air control but not for take-off/landing. Given these limitations, the precision of existing fully automated  systems for UAVs are still insufficient for landing and take-off, thereby necessitating the guidance of a human UAV operator during those phases.


To capture complex aerodynamic interactions without overly-constrained by conventional modeling assumptions, we take a machine-learning (ML) approach to build a black-box ground effect model using Deep Neural Networks (DNNs). However, incorporating such models into a UAV controller faces three key challenges. First, it is challenging to collect sufficient real-world training data, as DNNs are notoriously data-hungry. Second, due to high-dimensionality, DNNs can be unstable and generate unpredictable output, which makes the system  susceptible to instability in the feedback control loop. Third, DNNs are often difficult to analyze, which makes it difficult to design provably stable DNN-based controllers.

The aforementioned challenges pervade previous works using DNNs to capture high-order non-stationary dynamics.
For example,~\cite{abbeel2010autonomous, punjani2015deep} use DNNs to improve system ID of helicopter aerodynamics, but not for controller design. Other approaches aim to generate reference inputs or trajectories from DNNs \cite{bansal2016learning,li2017deep,zhou2017design,sanchez2018real}. However, these approaches can lead to challenging optimization problems  \cite{bansal2016learning}, or heavily rely on well-designed closed-loop controller and require a large number of labeled training data \cite{li2017deep, zhou2017design, sanchez2018real}.  A more classical approach of using DNNs is direct inverse control~\cite{balakrishnan1996neurocontrol, frye2014direct,suprijono2017direct} but the non-parametric nature of a DNN controller also makes it challenging to guarantee stability and robustness to noise. \cite{Berkenkamp2017SafeRL} proposes a provably stable model-based Reinforcement Learning method based on Lyapunov analysis, but it requires a potentially expensive discretization step and relies on the native Lipschitz constant of the DNN.  

\emph{Contributions.} In this paper, we propose a learning-based controller, \nn{}, to improve the precision of quadrotor landing with guaranteed stability.  Our approach directly learns the ground effect on coupled unsteady aerodynamics and vehicular dynamics. We use deep learning for system ID of residual dynamics and then integrate it with nonlinear feedback linearization control. 

We train DNNs with layer-wise spectrally normalized weight matrices. We prove that the resulting controller is globally exponentially stable under bounded learning errors. This is achieved by exploiting the Lipschitz bound of spectrally normalized DNNs.  It has earlier been shown that spectral normalization of DNNs leads to good generalization, i.e. stability in a learning-theoretic sense~\cite{miyato2018spectral}. It is intriguing that spectral normalization simultaneously guarantees stability both in a learning-theoretic and a control-theoretic sense.

We evaluate \nn{} on trajectory tracking of quadrotor during take-off, landing and cross-table maneuvers.  \nn{} is able to land a quadrotor much  more accurately than a \baselinelong{} with a pre-identified system. In particular, we show that compared to the baseline, \nn{} can decrease error in $z$ axis from \SI{0.13}{m} to $0$, mitigate $x$ and $y$ drifts by as much as 90\%, in the landing case. Meanwhile, \nn{} can decrease $z$ error from \SI{0.153}{m} to \SI{0.027}{m}, in the cross-table trajectory tracking task.\footnote{Demo videos: \url{https://youtu.be/FLLsG0S78ik}} We also demonstrate that the learned model can handle temporal dependency, and is an improvement over the steady-state theoretical models.

\section{Problem Statement: Quadrotor Landing}
Given quadrotor states as global position $\B{p} \in \mathbb{R}^3$, velocity $\B{v} \in \mathbb{R}^3$, attitude rotation matrix $R \in \mathrm{SO}(3)$, and body angular velocity $\bm{\omega} \in \mathbb{R}^3$, we consider the following dynamics:
\begin{subequations}
\begin{align}
\dot{\B{p}} &= \B{v}, &  
m\dot{\B{v}} &=m\B{g}+R\B{f}_u + \B{f}_a,\label{eq:pos_dynamics} \\ 
\dot{R}&=RS(\bm{\omega}), & 
J\dot{\bm{\omega}} &= J \bm{\omega} \times \bm{\omega}  + \bm{\tau}_u + \bm{\tau}_a,
\label{eq:att_dynamics}
\end{align}
\label{eq:dynamics}
\end{subequations}
where $m$ and $J$ are mass and inertia matrix of the system respectively, $S(\cdot)$ is skew-symmetric mapping. $\B{g} = [0, 0, -g]^\top$ is the gravity vector, $\B{f}_u = [0, 0, T]^\top$ and $\bm{\tau}_u = [\tau_x, \tau_y, \tau_z]^\top$ are the total thrust and body torques from four rotors predicted by a nominal model. We use $\bm{\eta} = [T, \tau_x, \tau_y, \tau_z]^\top$ to denote the output wrench. Typical quadrotor control input uses squared motor speeds $\B{u} = [n_1^2,n_2^2,n_3^2,n_4^2]^\top$, and is linearly related to the output wrench $\bm{\eta} = B_0 \B{u}$, with
\begin{equation}
B_0=
\left[\begin{smallmatrix}
c_T & c_T & c_T & c_T \\
0 & c_T l_{\mathrm{arm}} & 0 & - c_T l_{\mathrm{arm}}\\
-c_T l_{\mathrm{arm}} & 0 & c_T l_{\mathrm{arm}} & 0 \\
-c_Q & c_Q & -c_Q & c_Q
\end{smallmatrix}\right],
\label{eq:bmatrix}
\end{equation}
where $c_T$ and $c_Q$ are rotor force and torque coefficients, and $l_{\mathrm{arm}}$ denotes the length of rotor arm. 
The key difficulty of precise landing is the influence of unknown disturbance forces $\B{f}_a=[f_{a,x}, f_{a,y}, f_{a,z}]^\top$ and  torques $\bm{\tau}_a=[\tau_{a,x}, \tau_{a,y}, \tau_{a,z}]^\top$, which originate from complex aerodynamic interactions between the quadrotor and the environment. 


\emph{Problem Statement:} We aim to improve controller accuracy by learning the unknown disturbance forces $\B{f}_a$ and torques $\bm{\tau}_a$ in~\cref{eq:dynamics}. As we mainly focus on landing and take-off tasks, the attitude dynamics is limited and the aerodynamic disturbance torque $\bm{\tau}_a$ is bounded. Thus position dynamics~\cref{eq:pos_dynamics} and $\B{f}_a$ will our primary concern. We first approximate $\B{f}_a$ using a DNN with spectral normalization to guarantee its Lipschitz constant, and then incorporate the DNN in our exponentially-stabilizing controller. Training is done off-line, and the learned dynamics is applied in the on-board controller in real-time to achieve smooth landing and take-off.
\label{sec:prelim}

\section{Dynamics Learning using DNN}
\label{sec:network}
We learn the unknown disturbance force $\B{f}_a$ using a DNN with Rectified Linear Units (ReLU) activation. In general, DNNs equipped with ReLU converge faster during training, demonstrate more robust behavior with respect to changes in hyperparameters, and have fewer vanishing gradient problems compared to other activation functions such as  \textit{sigmoid} \cite{krizhevsky2012imagenet}.



\subsection{ReLU Deep Neural Networks}
A ReLU deep neural network represents the functional mapping from  the input $\B{x}$ to the output $f(\B{x},\bm{\theta})$, parameterized by the DNN  weights $\bm{\theta}={W^1,\cdots,W^{L+1}}$:
\begin{equation}
f(\B{x},\bm{\theta})=W^{L+1}\phi(W^L(\phi(W^{L-1}(\cdots\phi(W^1\B{x})\cdots)))),
\label{eq:relu}
\end{equation}
where the activation function $\phi(\cdot)=\max(\cdot,0)$ is called the element-wise ReLU function. ReLU is less computationally expensive than \textit{tanh} and \textit{sigmoid} because it involves simpler mathematical operations. However, deep neural networks are usually trained by first-order gradient based optimization, which is highly sensitive on the curvature of the training objective and can be unstable \cite{salimans2016weight}. To alleviate this issue, we apply the spectral normalization technique \cite{miyato2018spectral}.

\subsection{Spectral Normalization}
Spectral normalization stabilizes DNN training by constraining the Lipschitz constant of the objective function. Spectrally normalized DNNs have also been shown to generalize well \cite{bartlett2017spectrally}, which is an indication of stability in machine learning. Mathematically, the Lipschitz constant of a function  $\|f\|_{\text{Lip}}$ is defined as the smallest value such that
\[\forall \, \B{x}, \B{x}':\ \|f(\B{x})-f(\B{x}')\|_2/\|\B{x}-\B{x}'\|_2\leq \|f\|_{\text{Lip}}.\]
It is  known that the Lipschitz constant of a general differentiable function $f$ is the maximum spectral norm (maximum singular value) of its gradient over its domain $\|f\|_{\text{Lip}}= \sup_{\B{x}}\sigma(\nabla f(\B{x}))$.

The ReLU DNN in \cref{eq:relu} is a composition of functions.  Thus we can bound the Lipschitz constant of the network by constraining the spectral norm of each layer $g^l(\B{x})=\phi(W^l \B{x})$. Therefore, for a linear map $g(\B{x})=W\B{x}$, the spectral norm of each layer is given by $\|g\|_{\text{Lip}}=\sup_{\B{x}}\sigma(\nabla g(\B{x}))=\sup_{\B{x}}\sigma(W)=\sigma(W)$. 
%
%
Using the fact that the Lipschitz norm of ReLU activation function $\phi(\cdot)$ is equal to 1, with the inequality $\|g_1\circ g_2\|_{\text{Lip}}\leq\|g_1\|_{\text{Lip}}\cdot\|g_2\|_{\text{Lip}}$, we can find the following bound on $\|f\|_{\text{Lip}}$:
\begin{equation}
\|f\|_{\text{Lip}}\leq\|g^{L+1}\|_{\text{Lip}}\cdot\|\phi\|_{\text{Lip}}\cdots\|g^1\|_{\text{Lip}}=\prod_{l=1}^{L+1}\sigma(W^l).
\label{eq:upper_bound}
\end{equation}
In practice, we can apply spectral normalization to the weight matrices in each layer during training as follows:
\begin{equation}
\bar{W} = W / \sigma(W)\cdot \gamma^{\frac{1}{L+1}},
\label{eq:sn}
\end{equation}
where $\gamma$ is the intended Lipschitz constant for the DNN. The following lemma bounds the Lipschitz constant of a ReLU DNN   with  spectral normalization.
\begin{lemma}For  a multi-layer ReLU network $f(\B{x},\bm\theta)$, defined in~\cref{eq:relu} without an activation function on the output layer. Using spectral normalization, the Lipschitz constant of the entire network satisfies:
\begin{equation*}
\lVert f(\B{x}, \bar{\bm{\theta}}) \rVert_{\text{Lip}}\leq \gamma,
\label{eq:nn_lip}
\end{equation*}
with spectrally-normalized parameters $\bar{\bm{\theta}}={\bar{W} ^1,\cdots,\bar{W} ^{L+1}}$.
\label{lemma:sn}
\end{lemma}
\begin{proof}
As in \cref{eq:upper_bound}, the Lipschitz constant can be written as a composition of spectral norms over all layers. The proof follows from the spectral norms constrained as in \cref{eq:sn}.
\end{proof}
\subsection{Constrained Training}
We apply gradient-based optimization to train the ReLU DNN with a bounded Lipschitz constant. Estimating $\B{f}_a$ in (\ref{eq:dynamics}) boils down to optimizing the parameters $\bm\theta$ in the ReLU network in \cref{eq:relu}, given the observed value of $\B{x}$ and the target output. In particular, we want to control the Lipschitz constant of the ReLU network.

The optimization objective is as follows, where we minimize the prediction error with constrained Lipschitz constant:
\begin{align}
& \underset{\bm\theta}{\text{minimize}}
& & \sum_{t=1}^{T}\frac{1}{T}\|\B{y}_{t}-f(\B{x}_t,\bm\theta)\|_2 \nonumber \\ 
& \text{subject to}
& & \|f\|_{\text{Lip}} \leq \gamma.
\label{eq:opt}
\end{align}
In our case, $\B{y}_{t}$ is the observed disturbance forces and $\B{x}_t$ is the observed states and control inputs. According to the upper bound in~\cref{eq:upper_bound}, we can substitute the constraint  by minimizing the spectral norm of the weights in each layer.
%
We use stochastic gradient descent (SGD) to optimize \cref{eq:opt} and apply spectral normalization to regulate the weights. From Lemma \ref{lemma:sn}, the trained ReLU DNN has a Lipschitz constant.

\section{Neural Lander Controller Design}\label{sec:controller}
Our \nn{} controller for 3-D trajectory tracking is constructed as a nonlinear feedback linearization controller whose stability guarantees are obtained using the spectral normalizaion of the DNN-based ground-effect model. We then exploit the Lipschitz property of the DNN to solve for the resulting control input using fixed-point iteration.

\subsection{Reference Trajectory Tracking}
The position tracking error is defined as
$\tilde{\B{p}} = \B{p} - \B{p}_d$. Our controller uses a composite variable $\B{s} = 0$ as a manifold on which $\tilde{\B{p}}(t) \to 0$ exponentially:
\begin{equation}
\B{s} = \dot{\tilde{\B{p}}} + \Lambda \tilde{\B{p}} = \dot{\B{p}}-\B{v}_r
\label{eq:composite}
\end{equation}
with $\Lambda$ as a positive definite or diagonal matrix. Now the trajectory tracking problem is transformed to tracking a reference velocity $\B{v}_r = \dot{\B{p}}_d-\Lambda \tilde{\B{p}}$.

Using the methods described in~\cref{sec:network}, we define $\hat{\B{f}}_a(\bm{\zeta},\B{u})$ as the DNN approximation to the disturbance aerodynamic forces, with $\bm{\zeta}$ being the partial states used as input features to the network. We design the total desired rotor force $\B{f}_d$ as
\begin{align}
\B{f}_d = (R\B{f}_u)_d = \bar{\B{f}}_d - \hat{\B{f}}_a, \text{with} \ \
\bar{\B{f}}_d = m\dot{\B{v}}_r -  K_v \B{s} - m\B{g}.
\label{eq:pos_control}
\end{align}
Substituting~\cref{eq:pos_control} into~\cref{eq:dynamics}, the closed-loop dynamics would simply become $m\dot{\B{s}} +K_v\B{s}=\bm{\epsilon}$, with approximation error $\bm{\epsilon} = \B{f}_a-\hat{\B{f}}_a$. Hence, $\tilde{\B{p}}(t) \to \B{0}$ globally and exponentially with bounded error, as long as $\|\bm{\epsilon}\|$ is bounded~\cite{slotine1991applied,bandyopadhyay2016nonlinear,shi2018nonlinear}.

Consequently, desired total thrust $T_d$ and desired force direction $\hat{k}_d$ can be computed as
\begin{equation}
T_d = \B{f}_d \cdot \hat{k}, \ \text{and} \ \
\hat{k}_d = \B{f}_d / \left\lVert\B{f}_d\right\rVert,
\label{eq:desired_attitude}
\end{equation}
with $\hat{k}$ being the unit vector of rotor thrust direction (typically $z$-axis in quadrotors). Using $\hat{k}_d$ and fixing a desired yaw angle, desired attitude $R_d$ can be deduced~\cite{morgan2016swarm}. 
We assume that a nonlinear attitude controller uses the desired torque $\bm{\tau}_d$ from rotors to track $R_d(t)$. One such example is in~\cite{shi2018nonlinear}:
\begin{equation}
\bm{\tau}_d = J\dot{\bm{\omega}}_r - J\bm{\omega}\times\bm{\omega}_r - K_{\omega}(\bm{\omega} - \bm{\omega}_r),
\label{eq:att_control}
\end{equation}
where the reference angular rate $\bm{\omega}_r$ is designed similar to~\cref{eq:composite}, so that when $\bm{\omega}\to\bm{\omega}_r$, exponential trajectory tracking of a desired attitude $R_d(t)$ is guaranteed within some bounded error in the presence of bounded disturbance torques.

\subsection{Learning-based Discrete-time Nonlinear Controller}
From~\cref{eq:bmatrix,eq:desired_attitude,eq:att_control}, we can relate the desired wrench $\bm{\eta}_d = [T_d, \bm{\tau}_d^\top]^\top$ with the control signal $\B{u}$ through
\begin{equation}
B_0\B{u} = \bm{\eta}_d = 
\begin{bmatrix}
\left(\bar{\B{f}}_d - \hat{\B{f}}_a(\bm{\zeta},\B{u})\right) \cdot \hat{k}\\
\bm{\tau}_d
\end{bmatrix}.
\label{eq:control_nonaffine}
\end{equation}
Because of the dependency of $\hat{\B{f}}_a$ on $\B{u}$, the control synthesis problem here is non-affine. Therefore, we propose the following fixed-point iteration method for solving~\cref{eq:control_nonaffine}:
\begin{equation}
\B{u}_k = B_0^{-1} \bm{\eta}_d\left(\B{u}_{k-1}\right),
\label{eq:discrete_control}
\end{equation}
where $\B{u}_k$ and $\B{u}_{k-1}$ are the control input for current and previous time-step in the discrete-time controller. Next, we prove the stability of the system and convergence of the control inputs in~\cref{eq:discrete_control}.

\section{Nonlinear Stability Analysis}\label{sec:theory}
The closed-loop tracking error analysis provides a direct correlation on how to tune the neural network and controller parameter to improve control performance and robustness.
\subsection{Control Allocation as Contraction Mapping}
We first show that the control input $\B{u}_k$ converges to the solution of~\cref{eq:control_nonaffine} when all states are fixed. 
\begin{lemma} 
\label{thm:lem1}
Define mapping $\B{u}_{k} = \mathcal{F}(\B{u}_{k-1})$ based on~\cref{eq:discrete_control} and fix all current states:
\begin{equation}
\mathcal{F}(\B{u}) = B_0^{-1}
\begin{bmatrix}
\left(\bar{\B{f}}_d - \hat{\B{f}}_a(\bm{\zeta},\B{u})\right) \cdot \hat{k}\\
\bm{\tau}_d
\end{bmatrix}.
\label{eq:contraction_map}
\end{equation}
If $\hat{\B{f}}_a(\bm{\zeta},\B{u})$ is $L_a$-Lipschitz continuous, and $\sigma(B_0^{-1})\cdot L_a < 1$; then $\mathcal{F}(\cdot)$ is a contraction mapping, and $\B{u}_k$ converges to unique solution of $\B{u}^{*} = \mathcal{F}(\B{u}^{*})$.
\end{lemma}
\begin{proof}
$\forall\,\B{u}_1, \B{u}_2 \in \mathcal{U}$ with $\mathcal{U}$ being a compact set of feasible control inputs; and given fixed states as $\bar{\B{f}}_d$, $\bm{\tau}_d$ and $\hat{k}$, then:
\begin{align*}
\lVert\mathcal{F}(\B{u}_1) - \mathcal{F}(\B{u}_2) \rVert_2 &=\left\lVert B_0^{-1} \left(\hat{\B{f}}_a(\bm{\zeta},\bm{u}_1) - \hat{\B{f}}_a(\bm{\zeta},\bm{u}_2)\right) \right\rVert_2 \\
&\leq \sigma(B_0^{-1}) \cdot L_a \left\lVert \B{u}_1 - \B{u}_2\right\rVert_2.
\end{align*}
Thus, $\exists \ \alpha < 1, \text{s.t} \ \lVert\mathcal{F}(\B{u}_1) - \mathcal{F}(\B{u}_2) \rVert_2 < \alpha \left\lVert \B{u}_1 - \B{u}_2\right\rVert_2$. Hence, $\mathcal{F}(\cdot)$ is a contraction mapping. 
\end{proof}

\subsection{Stability of Learning-based Nonlinear Controller}
Before continuing to prove the stability of the full system, we make the following assumptions.
\begin{assumption}
The desired states along the position trajectory $\B{p}_d(t)$, $\dot{\B{p}}_d(t)$, and $\ddot{\B{p}}_d(t)$ are bounded.
\end{assumption}
\begin{assumption}
One-step difference of control signal satisfies $\left\lVert \B{u}_k - \B{u}_{k-1}\right\rVert \leq \rho \left\lVert \B{s} \right\rVert$ with a small positive $\rho$.
\end{assumption}

Here we provide the intuition behind this assumption. From~\cref{eq:contraction_map}, we can derive the following approximate relation with $\Delta(\cdot)_k = \lVert (\cdot)_{k} - (\cdot)_{k-1}\rVert$:
\begin{align*}
\Delta u_k &\leq \sigma(B_0^{-1})\big(L_a \Delta u_{k-1} + L_a \Delta \zeta_k \\
& \qquad + m \Delta \dot{v}_{r,k} + \lambda_{\max}(K_v)\Delta s_{k} + \Delta \tau_{d,k} \big).
\end{align*}
Because update rate of attitude controller ($ > \SI{100}{Hz}$) and motor speed control ($ > \SI{5}{kHz}$) are much higher than that of the position controller ($ \approx \SI{10}{Hz}$), in practice, we can safely neglect $\Delta s_k$, $\Delta \dot{v}_{r,k}$, and $\Delta \zeta_k$ in one update~(Theorem 11.1~\cite{khalil2002nonlinear}). Furthermore, $\Delta \tau_{d,k}$ can be limited internally by the attitude controller. It leads to:
\begin{equation*}
\Delta u_k \leq \sigma(B_0^{-1})\big(L_a \Delta u_{k-1} + c \big),
\end{equation*}
with $c$ being a small constant and $\sigma(B_0^{-1})\cdot L_a < 1$ from Lemma.~\ref{thm:lem1}, we can deduce that $\Delta u$ rapidly converges to a small ultimate bound between each position controller update. 

\begin{assumption}
The learning error of $\hat{\B{f}}_a(\bm{\zeta},\B{u})$ over the compact sets $\bm{\zeta} \in \mathcal{Z}$, $\B{u} \in \mathcal{U}$ is upper bounded by $\epsilon_m = \sup_{\bm{\zeta} \in \mathcal{Z}, \B{u} \in \mathcal{U}}\lVert \bm{\epsilon}(\bm{\zeta}, \B{u})\rVert$, where $\bm{\epsilon}(\bm{\zeta}, \B{u})=\B{f}_a(\bm{\zeta}, \B{u})-\hat{\B{f}}_a(\bm{\zeta}, \B{u})$.
\end{assumption}

DNNs have been shown to generalize well to the set of unseen events that are from almost the same distribution as training set~\cite{zhang2016understanding,he2016deep}. This empirical observation is also theoretically studied in order to shed more light toward an understanding of the complexity of these models~\cite{neyshabur2017pac,bartlett2017spectrally,dziugaite2017computing,neyshabur2017exploring}. 
Based on the above assumptions, we can now present our overall stability and robustness result.

\begin{theorem}
Under Assumptions 1-3, for a time-varying $\B{p}_d(t)$, the controller defined in~\cref{eq:pos_control,eq:discrete_control} with $\lambda_{\min}(K_v)>L_a\rho$ achieves exponential convergence of composite variable $\B{s}$ to error ball $\lim_{t\rightarrow\infty}\|\B{s}(t)\|=\epsilon_m/\left(\lambda_{\min}(K_v) - L_a \rho\right)$ with rate $(\left(\lambda_{\min}(K_v) - L_a \rho\right)/m$. And $\tilde{\B{p}}$ exponentially converges to error ball
\begin{equation}
\lim_{t\rightarrow\infty}\|\tilde{\B{p}}(t)\|=\frac{\epsilon_m}{\lambda_\mathrm{min}(\Lambda)(\lambda_{\min}(K_v) - L_a \rho)}
\label{eq:perrorinfty}
\end{equation}
with rate $\lambda_{\min}(\Lambda)$.
\end{theorem}
\begin{proof}
We begin the proof by selecting a Lyapunov function as $\mathcal{V(\B{s})} = \frac{1}{2}m\lVert \B{s}\rVert ^2$, then by applying the controller~\cref{eq:pos_control}, we get the time-derivative of $\mathcal{V}$:
\begin{align*}
\dot{\mathcal{V}} &= \B{s}^\top \big(-K_v \B{s} + \hat{\B{f}}_a(\bm{\zeta}_k, \B{u}_k) - \hat{\B{f}}_a(\bm{\zeta}_k,\B{u}_{k-1}) + \bm{\epsilon}(\bm{\zeta}_k, \B{u}_k) \big) \\
&\leq -\B{s}^\top K_v \B{s} + \lVert \B{s}\rVert(\lVert \hat{\B{f}}_a(\bm{\zeta}_k, \B{u}_k) - \hat{\B{f}}_a(\bm{\zeta}_k,\B{u}_{k-1}) \rVert + \epsilon_m)
\end{align*}
Let $\lambda =\lambda_{\min}(K_v)$ denote the minimum eigenvalue of the positive-definite matrix $K_v$. By applying the Lipschitz property of $\hat{\B{f}}_a$~\cref{lemma:sn} and Assumption 2, we obtain
\begin{align*}
\dot{\mathcal{V}} \leq -\frac{2\left(\lambda\!-\!L_a\rho\right)}{m}\mathcal{V} + \sqrt{\frac{2\mathcal{V}}{m}} \epsilon_m
\end{align*}
Using the Comparison Lemma\cite{khalil2002nonlinear}, we define $\mathcal{W}(t) = \sqrt{\mathcal{V}(t)}= \sqrt{m/2}\lVert\B{s}\rVert$ and $\dot{\mathcal{W}} = \dot{\mathcal{V}}/\left(2\sqrt{\mathcal{V}}\right)$ to obtain
\begin{align*}
\left\lVert\B{s}(t)\right\rVert
&\leq \left\lVert\B{s}(t_0)\right\rVert \exp\left(-\frac{\lambda-L_a\rho}{m}(t-t_0)\right) + \frac{\epsilon_m}{\lambda-L_a\rho} 
\end{align*} It can be shown that this leads to finite-gain $\mathcal{L}_p$ stability and input-to-state stability (ISS)~\cite{chung2013phase}. Furthermore, the hierarchical combination between $\B{s}$ and $\tilde{\B{p}}$ in \cref{eq:composite} results in  $\lim_{t\rightarrow\infty}\|\tilde{\B{p}}(t)\|=\lim_{t\rightarrow\infty}\|\B{s}(t)\|/\lambda_{\min}(\Lambda)$, yielding (\ref{eq:perrorinfty}).
\end{proof}


\section{Experiments}
\label{sec:exp}
In our experiments, we evaluate both the generalization performance of our DNN as well as the overall control performance of \nn{}.
The experimental setup is composed of a motion capture system with 17 cameras, a WiFi router for communication, and an Intel Aero drone, weighing 1.47 kg with an onboard Linux computer (2.56 GHz Intel Atom x7 processor, 4 GB DDR3 RAM).  We retrofitted the drone with eight reflective infrared markers for accurate position, attitude and velocity estimation at 100Hz. The Intel Aero drone and the test space are shown in \cref{fig:training_data_and_drone}(a).

\subsection{Bench Test}
\label{sec:bench}
To identify a good nominal model, we first measured the mass, $m$, diameter of the rotor, $D$, the air density, $\rho$, gravity, $g$. Then we performed bench test to determine the thrust constant, $c_T$, as well as the non-dimensional thrust coefficient $C_T=\frac{c_T}{\rho D^4}$.
Note that $C_T$ is a function of propeller speed $n$, and here we picked a nominal value at $n=\SI{2000}{RPM}$ .

\subsection{Real-World Flying Data and Preprocessing}

To estimate the disturbance force $\B{f}_a$, an expert pilot manually flew the drone at different heights, and we collected training data consisting of sequences of state estimates and control inputs  $\{(\B{p},\B{v},R,\B{u}),\B{y}\}$ where $\B{y}$ is the observed value of $\B{f}_a$.
\begin{figure}[!t]
\includegraphics[width=0.9\columnwidth]{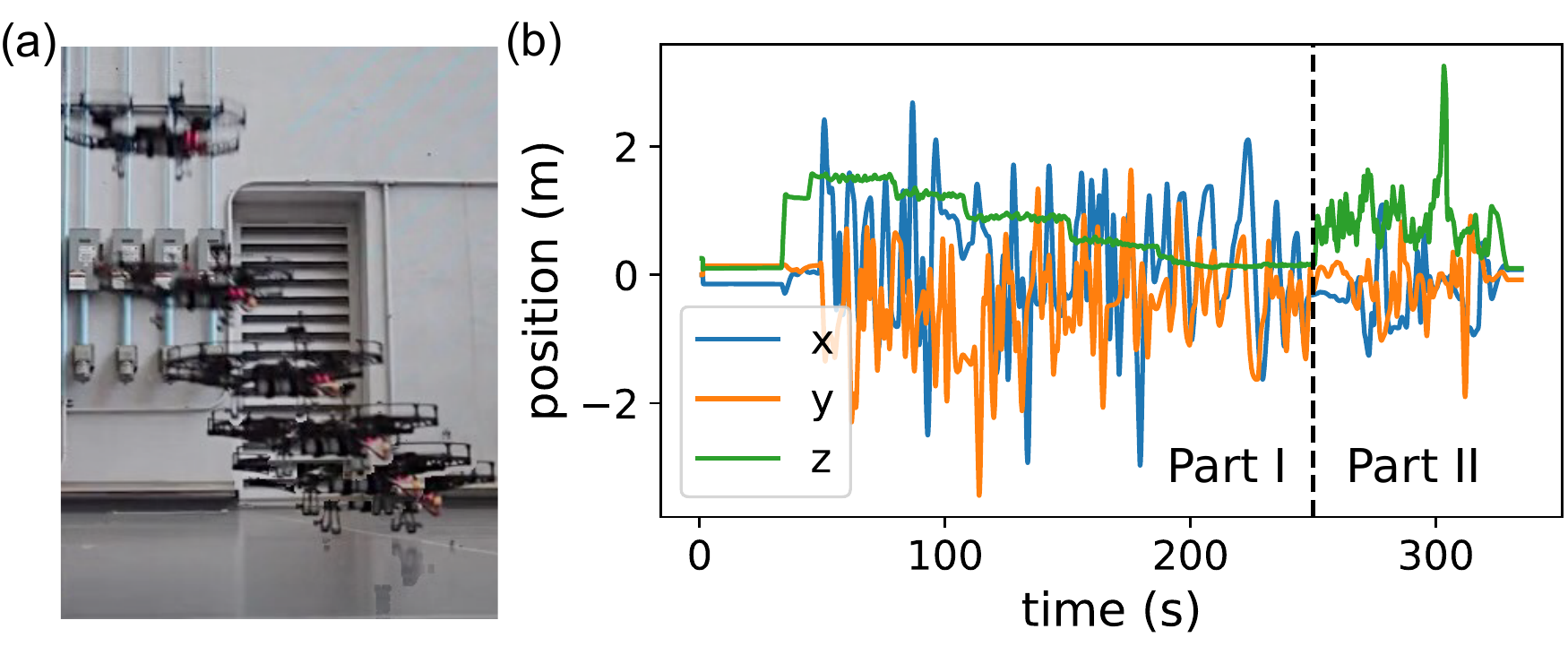}
\vspace{-0.05in}
\caption{(a) Intel Aero drone; (b) Training data trajectory. Part I ($0$ to \SI{250}{s}) contains maneuvers at different heights (\SI{0.05}{m} to \SI{1.50}{m}). Part II (\SI{250}{s} to \SI{350}{s}) includes random $x$, $y$, and $z$ motions for maximum state-space coverage.}
\label{fig:training_data_and_drone}
\end{figure}
We utilized the relation $\B{f}_a=m\dot{\B{v}}-m\B{g}-R\B{f}_u$ from \cref{eq:dynamics} to calculate  $\B{f}_a$, where $\B{f}_u$ is calculated based on the nominal $c_T$ from the bench test in~\cref{sec:bench}. Our training set is a single continuous trajectory with varying heights and velocities. The trajectory has two parts shown in \cref{fig:training_data_and_drone}(b). We aim to learn the ground effect through Part I of the training set, and other aerodynamics forces such as air drag through Part II.


\begin{figure}
\centering
\includegraphics[width=0.9\columnwidth]{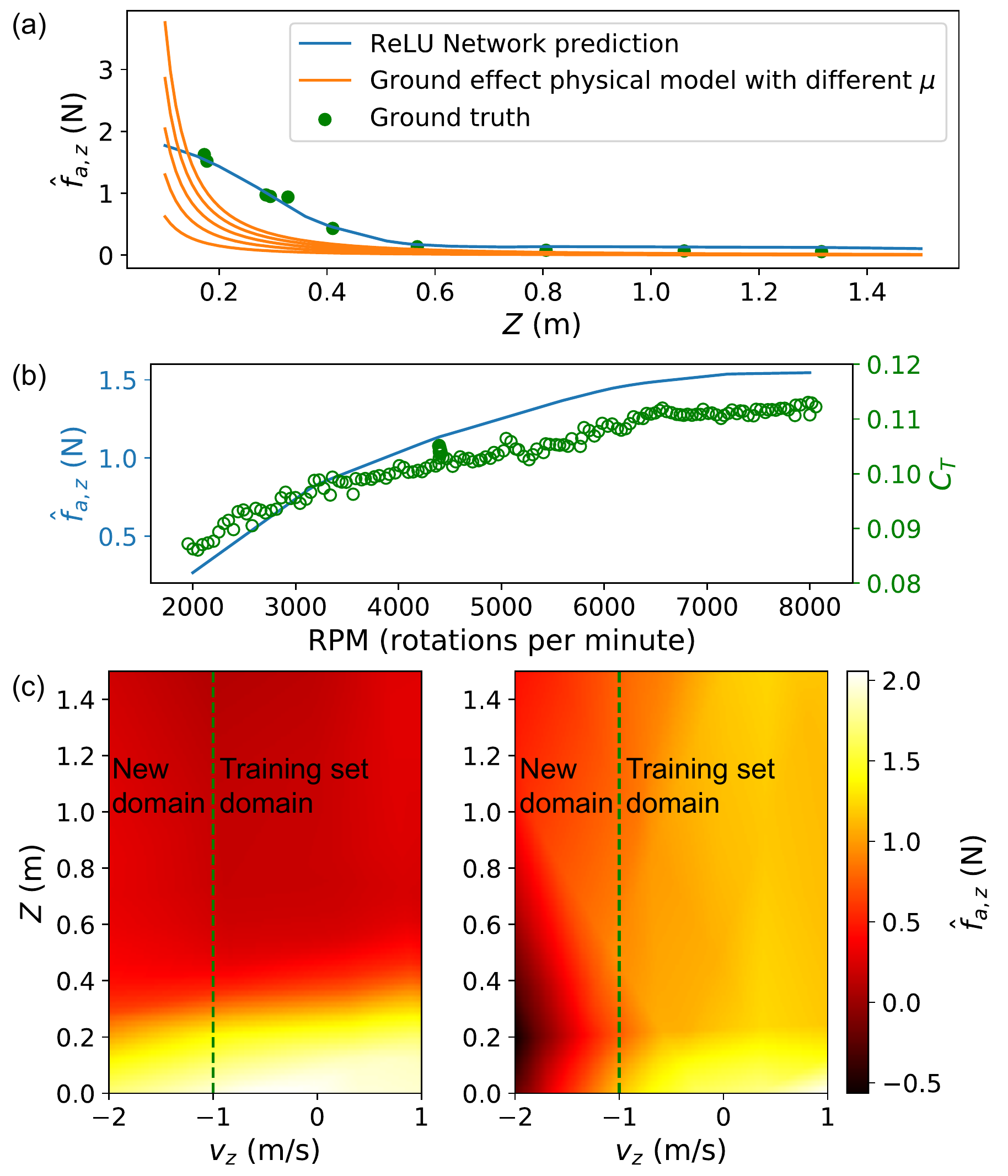}
\caption{(a) Learned $\hat{f}_{a,z}$ compared to the ground effect model with respect to height $z$, with $v_z=v_x=v_y=\SI{0}{m/s}$, $R=I$, $\B{u}=\SI{6400}{RPM}$. Ground truth points are from hovering data at different heights. (b) Learned $\hat{f}_{a,z}$ with respect to rotation speed $n$ ($z=\SI{0.2}{m}$, $v_z=\SI{0}{m/s}$), compared to $C_T$ measured in the bench test. (c) Heatmaps of learned $\hat{f}_{a,z}$ versus $z$ and $v_z$. (Left) ReLU network with spectral normalization. (Right) ReLU network without spectral normalization.}
\label{fig:pred_all}
\end{figure}

\subsection{DNN Prediction Performance}
\label{sec:relu_training}
We train a deep  ReLU network $\hat{\B{f}}_a(\bm{\zeta},\B{u})=\hat{\B{f}}_a(z,\B{v},R,\B{u})$, with $z$, $\B{v}$, $R$, $\B{u}$ corresponding to global height, global velocity, attitude, and control input.  We build the ReLU network using PyTorch~\cite{paszke2017automatic}. Our ReLU network consists of four fully-connected hidden layers, with input and the output dimensions 12 and 3, respectively. We use spectral normalization~\cref{eq:sn} to constrain the Lipschitz constant of the DNN. 

We compare the near-ground estimation accuracy our DNN model with existing 1D steady ground effect model~\cite{cheeseman1955effect,danjun2015autonomous}:
\begin{align}
T(n,z)=\frac{n^2}{1-\mu(\frac{D}{8z})^2}c_T(n)=n^2c_T(n_0)+\bar{f}_{a,z},
\label{eq:gt}
\end{align}
where $T$ is the thrust generated by propellers, $n$ is the rotation speed,  $n_0$ is the idle RPM, and $\mu$ depends on the number and the arrangement of propellers ($\mu=1$ for a single propeller, but must be tuned for multiple propellers).
Note that $c_T$ is a function of $n$. Thus, we can derive $\bar{f}_{a,z}(n,z)$ from $T(n,z)$.

\cref{fig:pred_all}(a) shows the comparison between the estimated $\B{f}_a$ from DNN and the theoretical ground effect model \cref{eq:gt} at different $z$ (assuming $T=mg$ when $z=\infty$). We can see that our DNN can achieve much better estimates than the theoretical ground effect model. We further investigate the trend of $\bar{f}_{a,z}$ with respect to the rotation speed $n$. 
\cref{fig:pred_all}(b) shows the learned $\hat{f}_{a,z}$ over the rotation speed $n$ at a given height, in comparison with the $C_T$ measured from the bench test. We observe that the increasing trend of the estimates $\hat{f}_{a,z}$ is consistent with bench test results for $C_T$.

To understand the benefits of SN, we compared $\hat{f}_{a,z}$ predicted by the DNNs trained both with and without SN as shown in~\cref{fig:pred_all}(c).
Note that $v_z$ from \SI{-1}{m/s} to \SI{1}{m/s} is covered in our training set, but \SI{-2}{m/s} to \SI{-1}{m/s} is not. We observe the following differences:
\begin{enumerate}
    \item  Ground effect: $\hat{f}_{a,z}$ increases as $z$ decreases, which is also shown in \cref{fig:pred_all}(a).
    \item Air drag: $\hat{f}_{a,z}$ increases as the drone goes down ($v_z<0$) and it decreases as the drone goes up ($v_z>0$).
    \item Generalization: the spectral normalized DNN is much smoother and can also generalize to new input domains not contained in the training set. 
\end{enumerate}
In \cite{bartlett2017spectrally}, the authors theoretically show that spectral normalization can provide tighter generalization guarantees on unseen data, which is consistent with our empirical observation. We will connect generalization theory more tightly with our robustness guarantees in the future.

\begin{figure}[!t]
\centering{
	{
    	\includegraphics[width=0.9\columnwidth]{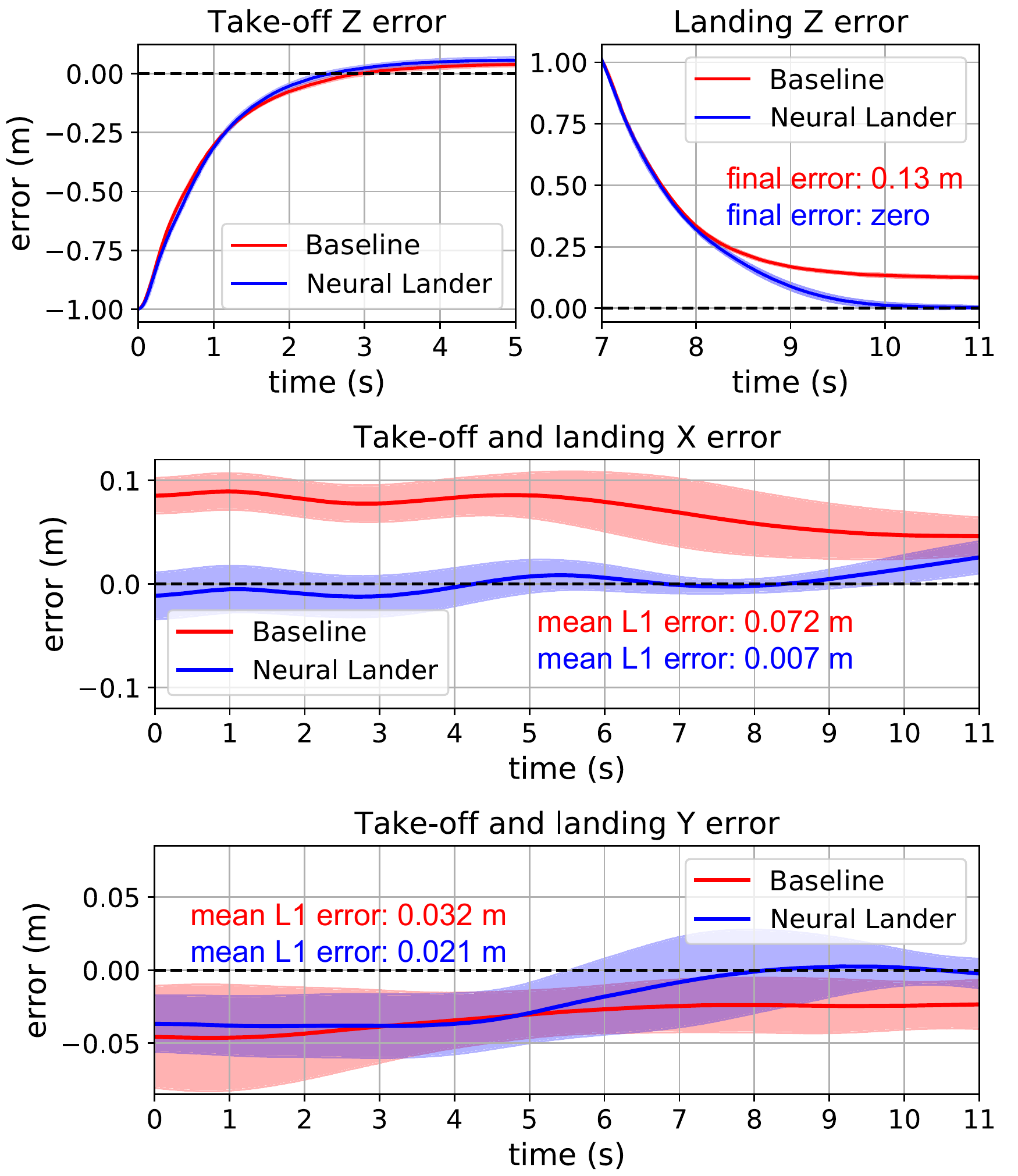}
    }
}
\caption{\baseline{} and {\nn} performance in take-off and landing. Means (solid curves) and standard deviations (shaded areas) of 10 trajectories.}
\label{fig:1-d}
\end{figure}



\subsection{Baseline Controller}

We compared the \nn{} with a \baselinelong{}. We implemented both a \baseline{} similar to~\cref{eq:composite,eq:pos_control} with $\hat{\B{f}}_a\equiv0$, as well as an integral controller variation with $\B{v}_r = \dot{\B{p}}_d-2\Lambda \tilde{\B{p}}-\Lambda^2\int^{t}_{0}\tilde{\B{p}}(\tau)d\tau$. Though an integral gain can cancel steady-state error during set-point regulation, our flight results showed that the performance can be sensitive to the integral gain, especially during trajectory tracking. This can be seen in the demo video.\footnote{Demo videos: \url{https://youtu.be/FLLsG0S78ik}}

\subsection{Setpoint Regulation Performance}
\begin{figure}[!t]
\centering{
	{
    	\includegraphics[width=0.9\columnwidth]{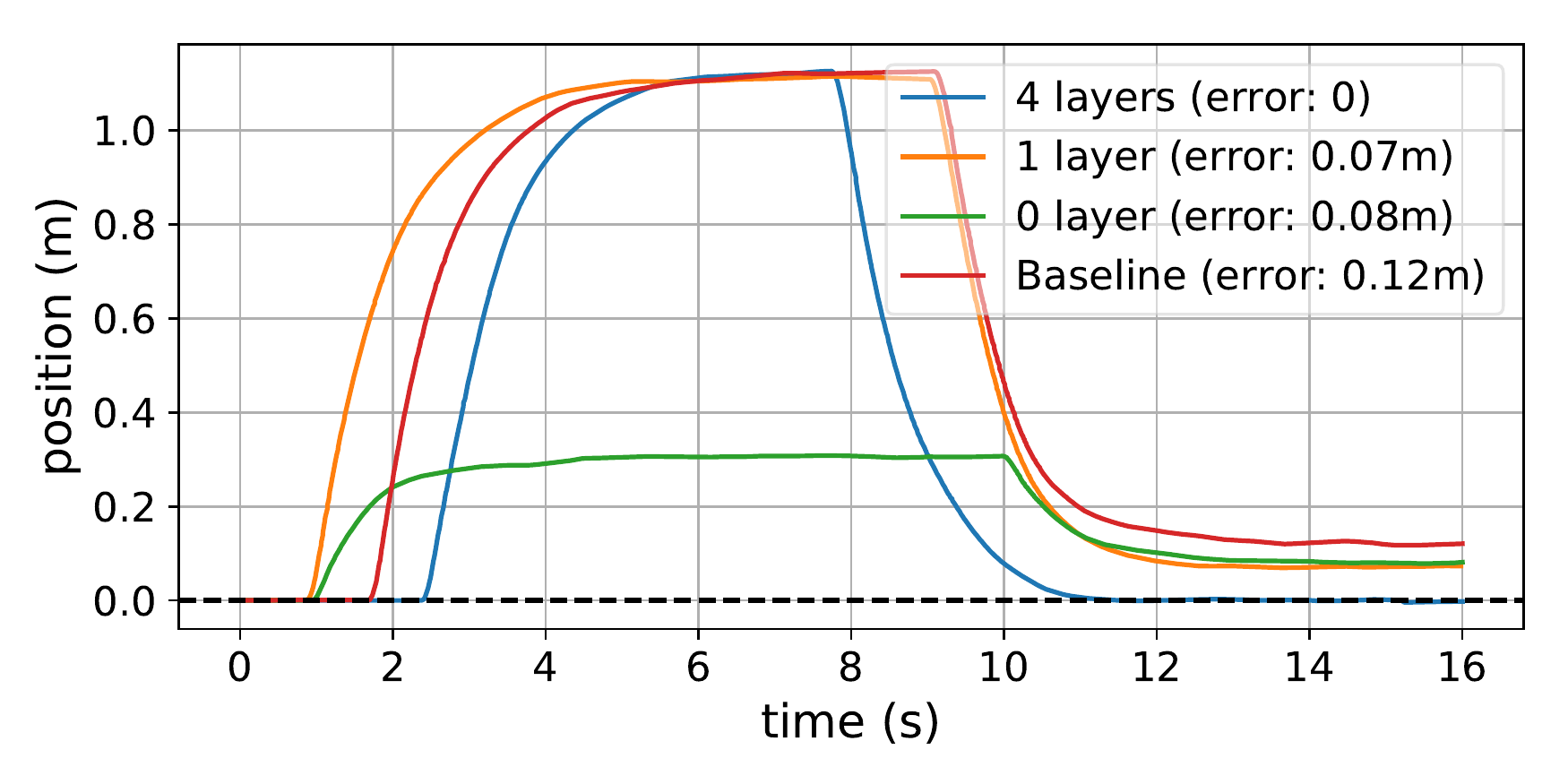}
    }
}
\caption{{\nn} performance in take-off and landing with different DNN capacities. 1 layer means $\hat{\B{f}}_a=A\B{x}+\B{b}$; 0 layer means $\hat{\B{f}}_a=\B{b}$; Baseline means $\hat{\B{f}}_a\equiv0$.}
\label{fig:different_model}
\end{figure}

First, we tested the two controllers' performance in take-off/landing, by commanding position setpoint $\B{p}_d$, from $(0,0,0)$, to $(0,0,1)$, then back to $(0,0,0)$, with $\dot{\B{p}_d} \equiv 0$. 
From~\cref{fig:1-d}, we can conclude that there are two main benefits of our  {\nn}. (a) {\nn} can  control the drone to precisely and smoothly land on the ground surface while the  \baseline{} struggles to achieve $0$ terminal height due to the ground effect. (b) {\nn} can mitigate drifts in $x-y$ plane, as it also learned about additional aerodynamics such as air drag.

Second, we tested \nn{} performance with different DNN capacities. \cref{fig:different_model} shows that compared to the baseline ($\hat{\B{f}}_a\equiv0$), 1 layer model could decrease $z$ error but it is not enough to land the drone. 0 layer model generated significant error during take-off. 

In experiments, we observed the \nn{} without spectral normalization can even result in unexpected controller outputs leading to crash, which empirically implies the necessity of SN in training the DNN and designing the controller.

\subsection{Trajectory Tracking Performance}
\begin{figure}[!t]
\centering
\includegraphics[width=0.4\textwidth, angle=0]{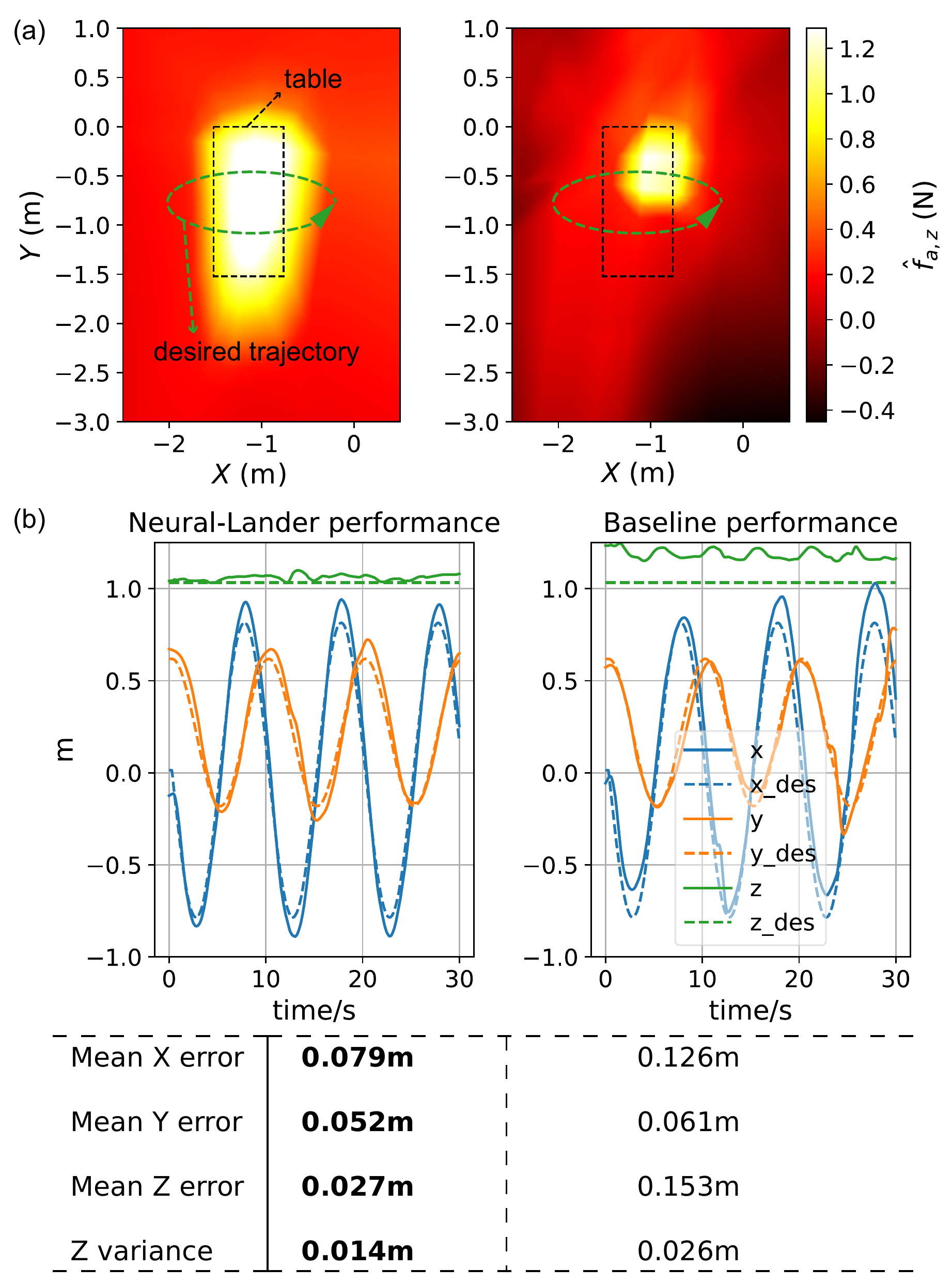}
\caption{(a) Heatmaps of learned $\hat{f}_{a,z}$ versus $x$ and $y$, with other inputs fixed. (Left) ReLU network with spectral normalization. (Right) ReLU network without spectral normalization. (b) Tracking performance and statistics.}
\vspace{-0.1in}
\label{fig:table_fa_and_tracking}
\end{figure}

To show that our algorithm can handle more complicated environments where physics-based modelling of dynamics would be substantially more difficult, we devise a task of tracking an elliptic trajectory very close to a table with a period of 10 seconds shown in~\cref{fig:table_fa_and_tracking}. The trajectory is partially over the table with significant ground effects, and a sharp transition to free space at the edge of the table. We compared the performance of both \nn{} and \baseline{} on this test. 

In order to model the complex dynamics near the table, we manually flew the drone in the space close to the table to collect another data set.  We trained a new ReLU DNN model  with $x$-$y$ positions as additional input features: $\hat{\B{f}}_a(\B{p},\B{v},R,\B{u})$.  Similar to the setpoint experiment, the benefit of spectral normalization can be seen in \cref{fig:table_fa_and_tracking}(a), where only the spectrally-normalized DNN exhibits a clear table boundary.

\cref{fig:table_fa_and_tracking}(b) shows that \nn{} outperformed the \baseline{} for tracking the desired position trajectory in all $x$, $y$, and $z$ axes. Additionally, \nn{} showed a lower variance in height, even at the edge of the table, as the controller captured the changes in ground effects when the drone flew over the table.

In summary, the experimental results with multiple ground interaction scenarios show that much smaller tracking errors are obtained by \nn{}, which is essentially the nonlinear tracking controller with feedforward cancellation of a spectrally-normalized DNN.
\section{Conclusions}
\label{sec:con}
In this paper, we present \nn{}, a deep learning based nonlinear controller with guaranteed stability for precise quadrotor landing. Compared to the \baseline{}, \nn{} is able to significantly improve control performance. 
The main benefits are (1) our method can learn from coupled unsteady aerodynamics and vehicle dynamics to provide more accurate estimates than theoretical ground effect models, (2) our model can capture both the ground effect and other non-dominant aerodynamics and outperforms the conventional controller in all axes ($x$, $y$ and $z$), and (3) we provide rigorous theoretical analysis of our method and guarantee the stability of the controller, which also implies generalization to unseen domains. 

Future work includes further generalization of the capabilities of \nn{} handling unseen state and disturbance domains, such as those generated by a wind fan array.



\section*{ACKNOWLEDGEMENT}
The authors thank Joel Burdick, Mory Gharib and Daniel Pastor Moreno. The work is funded in part by Caltech's Center for Autonomous Systems and Technologies and Raytheon Company.


\bibliographystyle{IEEEtran}
\bibliography{IEEEabrv,ref}

\end{document}